\documentclass{article}
\usepackage{ijcai23}

\usepackage{times}
\usepackage{soul}
\usepackage{url}
\usepackage[hidelinks]{hyperref}
\usepackage[utf8]{inputenc}
\usepackage[small]{caption}
\usepackage{graphicx}
\usepackage{amsmath}
\usepackage{amsthm}
\usepackage{booktabs}
\usepackage{algorithm}
\usepackage{algorithmic}
\usepackage[switch]{lineno}

\newtheorem{theorem}{Theorem}
\newtheorem*{theorem*}{Theorem}
\newtheorem{cor}{Corollary}
\newtheorem{remark}{Remark}

\title{Solving Long-run Average Reward Robust MDPs via Stochastic Games}

\author{
Krishnendu Chatterjee$^1$\and
Ehsan Kafshdar Goharshady$^1$\and
Mehrdad Karrabi$^1$\and\\
Petr Novotn\'y$^2$\and
\DJ or\dj e \v{Z}ikeli\'c$^3$\thanks{Part of the work done while the author was at the Institute of Science and Technology Austria (ISTA).}
\affiliations
$^1$Institute of Science and Technology Austria (ISTA), Austria\\
$^2$Masaryk University, Czech Republic\\
$^3$Singapore Management University, Singapore
\emails
\{krishnendu.chatterjee, ehsan.goharshady, mehrdad.karrabi\}@ist.ac.at, \\
petr.novotny@fi.muni.cz, dzikelic@smu.edu.sg
}

\usepackage{enumitem}
\usepackage{paralist}
\usepackage{mathtools, amssymb}
\usepackage{multicol}
\usepackage{xcolor}
\usepackage{comment}
\usepackage[T1]{fontenc}
\usepackage{booktabs}

\newcommand{\rew}{r}
\newcommand{\Rset}{\mathbb{R}}

\newcommand{\dist}{\Delta}
\newcommand{\tran}{\delta}
\newcommand{\uncert}{\mathcal{P}}

\newcommand{\probm}{\mathbb{P}}
\newcommand{\expv}{\mathbb{E}}

\newcommand{\todoE}[1]{\textbf{\textcolor{red}{EG: TODO:  \emph{#1}}}}
\newcommand{\limavg}{\textsc{LimAvg}}
\newcommand{\val}{\mathit{Val}}
\newcommand{\rmdp}{\mathcal{M}}
\newcommand{\payoff}{\textsc{Payoff}}
\newcommand{\disc}{\textsc{Disc}}

\begin{document}

\maketitle

\begin{abstract}
    Markov decision processes (MDPs) provide a standard framework for sequential decision making under uncertainty. However, MDPs do not take uncertainty in transition probabilities into account. Robust Markov decision processes (RMDPs) address this shortcoming of MDPs by assigning to each transition an uncertainty set rather than a single probability value.
	In this work, we consider polytopic RMDPs in which all uncertainty sets are polytopes and study the problem of solving long-run average reward polytopic RMDPs. We present a novel perspective on this problem and show that it can be reduced to solving long-run average reward turn-based stochastic games with {\em finite state and action spaces}. This reduction allows us to derive several important consequences that were hitherto not known to hold for polytopic RMDPs.
    First, we derive new computational complexity bounds for solving long-run average reward polytopic RMDPs, showing for the first time that the threshold decision problem for them is in $\textsc{NP} \cap \textsc{coNP}$ and that they admit a randomized algorithm with sub-exponential expected runtime. Second, we present Robust Polytopic Policy Iteration (RPPI), a novel policy iteration algorithm for solving long-run average reward polytopic RMDPs. Our experimental evaluation shows that RPPI is much more efficient in solving long-run average reward polytopic RMDPs compared to state-of-the-art methods based on value iteration. 
\end{abstract}

\section{Introduction}\label{sec:intro}

\noindent{\bf Robust Markov decision processes.} Markov decision processes (MDPs) are widely adopted as a framework for sequential decision making in environments that exhibit stochastic uncertainty~\cite{Puterman94}. The goal of solving an MDP is to compute a policy which maximizes expected payoff with respect to some reward objective. However, classical algorithms for solving MDPs assume that transition probabilities are known and available. While this is plausible from a theoretical perspective, in practice this assumption is not always justified -- transition probabilities are typically inferred from data and as such their estimates come with a certain level of uncertainty. This has motivated the study of {\em robust Markov decision processes (RMDPs)}~\cite{NilimG03,Iyengar05}. RMDPs differ from MDPs in that they only assume knowledge of some uncertainty set containing true transition probabilities. The goal of solving RMDPs is to compute a policy which maximizes the {\em worst-case expected payoff} with respect to all possible choices of transition probabilities belonging to the uncertainty set. The advantage of considering RMDPs rather than MDPs for planning and sequential decision making tasks is that solving RMDPs provides policies that are {\em robust} to potential uncertainty in the underlying environment model.

\smallskip\noindent{\bf Long-run average reward RMDPs.} Recent years have seen significant interest in solving RMDPs. Two most fundamental reward objectives defined over unbounded or infinite time horizon trajectories are the {\em discounted-sum reward} and the {\em long-run average reward}. While there are many works on solving discounted-sum reward RMDPs (see an overview of related work below), the problem of solving long-run average reward RMDPs remains comparatively less explored. However, long-run average reward is, along with discounted-sum, one of the two fundamental objectives used in agent planning and reinforcement learning~\cite{Puterman94,SuttonMSM99}. Unlike discounted-sum, which places more emphasis on earlier rewards and prefixes of a run, long run average reward focuses on long-term sustainability of the reward signal.

\smallskip\noindent{\bf Prior work and challenges.} Most prior work on solving long-run average reward RMDPs focuses on special instances of RMDPs in which uncertainty sets are either products of intervals (i.e.~$L^\infty$-balls) or $L^1$-balls around some pivot probability distribution. These works typically focus on studying {\em Blackwell optimality} in RMDPs. In standard MDPs, Blackwell optimality~\cite{Puterman94} is a classical result which says that there exists some threshold discount factor $\gamma^\ast \in [0,1)$ such that the same policy maximizes expected discounted-sum reward for any discount factor $\gamma \in [\gamma^*,1)$ while also maximizing the long-run average reward in the MDP. This allows using the methods for discounted-sum reward MPDs to solve long-run average reward MDPs.

The work of~\cite{TewariB07} considers RMDPs in which uncertainty sets are products of intervals (also called interval MDPs~\cite{GivanLD00}), and is the first to establish Blackwell optimality for interval MDPs. It also proposes a value iteration algorithm for solving long-run average reward interval MDPs. The works~\cite{ClementP23,GoyalG23} establish Blackwell optimality for RMDPs with general polytopic uncertainty sets. Moreover, the work~\cite{ClementP23} designs value and policy iteration algorithms for solving long-run average reward RMDPs with interval or $L^1$-uncertainty sets. Finally, the recent work of~\cite{WangVAPZ23} establishes Blackwell optimality and proposes a value iteration algorithm for a general class of RMDPs in which uncertainty sets are only assumed to be compact, but in which the RMDP is assumed to be unichain.

While these works present significant advances in solving long-run average reward RMDPs, our understanding of this problem still suffers from a few limitations:
\begin{compactitem}
    \item {\em Computational complexity.} While prior works focus on studying Blackwell optimality in RMDPs and algorithms for solving them, no upper bound on the computational complexity of solving this problem is known.
    \item {\em Structural or uncertainty restrictions.} The value iteration algorithm of~\cite{WangVAPZ23} requires the RMDP to be unichain. On the other hand, other algorithms discussed above are applicable to RMDPs with either interval or $L^1$-uncertainty sets.
    \item {\em Policy iteration.} In classical MDP and stochastic games~\cite{Puterman94,HansenMZ13} and in discounted-sum RMDPs~\cite{SinhaG16,HoPW21}, the number of iterations of policy iteration is known to be more efficient than value iteration. However, for RMDPs with uncertainty sets that are not intervals or $L^1$-balls, the only available algorithm is that of~\cite{WangVAPZ23} which is based on value iteration. 
\end{compactitem}

\smallskip\noindent{\bf Our contributions.} This work focuses on computational complexity aspects and efficient algorithms for solving long-run average reward RMDPs and our goal is to address the limitations discussed above. To that end, we consider {\em polytopic RMDPs} in which all uncertainty sets are assumed to be polytopes and which unify and strictly subsume RMDPs with interval or $L^1$-uncertainty sets. We present a {\em new perspective} on the problem of solving long-run average reward polytopic RMDPs by showing that it admits a {\em linear-time reduction} to the problem of solving long-run average reward {\em turn-based stochastic games (TBSGs)} in which both state and action sets are finite. The reduction is linear in the size of the RMDP and the number of vertices of uncertainty polytopes. While the connection between RMDPs and TBSGs has been informally mentioned in several prior works, e.g.~\cite{NilimG03,TesslerEM19,GoyalG23}, we are not aware that any prior work has actually formalized this reduction. The reduction formalization turns out to be technically non-trivial and we provide a first formal proof of the correctness of the reduction. The reduction formalization is the main technical contribution of our work.

While definitely compelling, the reduction is not only of theoretical interest. TBSGs have been widely studied within computational game theory and formal methods communities~\cite{Shapley:stochastic-games,FV:book,Kucera:games-chapter} and this reduction allows us to leverage a plethora of results on TBSGs in order to establish new results on long-run average reward RMDPs that were hitherto not know to be true. In particular, our results allow us to overcome all the limitations discussed above. Our main contributions are as follows:
\begin{compactenum}
    \item We establish novel {\em complexity results} on solving long-run average reward polytopic RMDPs, showing for the first time that the threshold decision problem for them is in $\textsc{NP} \cap \textsc{coNP}$ and that they admit a randomized algorithm with sub-exponential expected runtime. Moreover, for RMDPs with constant number of actions and treewidth of the underlying graph, the problem admits an algorithm with quasi-polynomial runtime.
    \item By utilizing Blackwell optimality for TBSGs and the fact that policy iteration has already been shown to be sound for discounted-sum reward TBSGs, we propose {\em Robust Polytopic Policy Iteration (RPPI)}, our novel algorithm for solving long-run average reward polytopic RMDPs. Our algorithm {\em does not impose structural restrictions} on the RMDP such as unichain or aperiodic. 
    \item {\em Experimental evaluation.} We implement RPPI and experimentally compare it against the value iteration algorithm of~\cite{WangVAPZ23}. We show that RPPI achieves significant efficiency gains on unichain RMDPs while also being practically applicable to RMDPs without such structural restrictions.
\end{compactenum}

\smallskip\noindent{\bf Related work.} The related work on solving long-run average reward RMDPs was discussed above, hence we omit repetition. In what follows, we overview a larger body of existing works on solving discounted-sum reward RMDPs.  Existing approaches can be classified into model-based and model-free. Model-based approaches assume the full knowledge of the uncertainty set and compute optimal policies with exact values on the worst-case expected payoff that they achieve, see e.g.~\cite{NilimG03,Iyengar05,WiesemannKR13,KaufmanS13,TamarMX14,GhavamzadehPC16,HoPW21}. These works also provide deep theoretical insight into RMDPs and show that many classical results on MDPs transfer to the RMDP setting. Model-free approaches, on the other hand, do not assume the full knowledge of the uncertainty set. Rather, they learn the uncertainty sets from data and as such provide more scalable and practical algorithms, although without guarantees on optimality of learned policies~\cite{RoyXP17,TesslerEM19,WangZ21,WangZ22c,HoPW18,yang2022toward}.

In the final days of preparing this submission, we found a concurrent work of~\cite{grand2023beyond} which also explores the connection between long-run average reward RMDPs and TBSGs in order to study Blackwell optimality. Our focus in this work is on computational complexity aspects and efficient algorithms for solving this problem.  Furthermore, we also establish that our reduction is efficient, along with its formal correctness, that allows transfer of algorithms from the stochastic games literature to RMDPs.

\section{Preliminaries and Models}


For a set \( X \), we denote by \( \dist(X) \) the set of all probability distributions over \( X \). 


\smallskip\noindent{\bf MDPs.} A (finite-state, discrete-time) \emph{Markov decision process (MDP)}~\cite{Puterman94} is a tuple \( (S, A, \tran, \rew, s_0) \), where \( S \) is a finite set of \emph{states}, \( A \) is a finite set of \emph{actions}, \( \tran \colon S \times A \rightarrow \dist(S) \) is a probabilistic \emph{transition function}, \( \rew \colon S\times A \rightarrow \Rset\) is a \emph{reward function,} and \(s_0 \in S\) is the {\em initial state}.
We abbreviate the transition probability \( \delta(s,a)(s') \) by \( \delta(s' \mid s,a) \).

The dynamics of an MDP are defined through the notion of a policy. A {\em policy} is a map $\sigma: (S\times A)^{\ast} \times S \rightarrow \dist(A)$ which to each finite history of state-action pairs ending in a state assigns the probability distribution over the actions. Under a policy \( \sigma \), the MDP evolves as follows. The MDP starts in the initial state $s_0$. Then, in any time step \( t \in \{0,1,2,\ldots\} \), based on the current history $h_t = s_0,a_0,\dots,s_{t-1},a_{t-1},s_t$ the policy \( \sigma \) samples an action \( a_t \) to play according to the probability distribution $\sigma(h_t)$. The agent then receives the immediate reward \( r_t = \rew(s_t,a_t) \) and the MDP transitions into the next state \( s_{t+1} \) sampled according to the probability distribution \( \tran(s_t,a_t) \). The process continues in this manner \emph{ad infinitum.} A policy is said to be:
\begin{compactitem}
    \item {\em positional (or memoryless)}, if the choice of the probability distribution over actions depends only on the last state in the history, i.e.~if $\sigma(s_0,a_0,\dots,s_{t-1},a_{t-1},s_t) = \sigma(s_0',a_0',\dots,s_{t-1}',a_{t-1}',s_t')$ whenever $s_t = s_t'$;
    \item {\em pure}, if for every history $h$ the probability distribution $\sigma(h)$ assigns probability~$1$ to a single action in $A$.
\end{compactitem}



\smallskip\noindent{\bf Robust MDPs.} Robust MDPs generalize MDPs in that they prescribe to each state-action pair an {\em uncertainty set}, which contains a set of probability distributions over the states in $S$. Thus, MDPs present a special case of robust MDPs in which the uncertainty set prescribed to each state-action pair contains only a single probability distribution over the states.

Formally, a \emph{robust MDP (RMDP)}~\cite{NilimG03,Iyengar05} is a tuple \( (S, A, \uncert, \rew, s_0) \), where \( S \) is a finite set of \emph{states}, \( A \) is a finite set of \emph{actions}, \( \uncert \subseteq \dist(S)^{S \times A} \) is an {\em uncertainty set}, \( \rew \colon S\times A \rightarrow \Rset\) is a \emph{reward function,} and \(s_0 \in S\) is the initial state.

The intuition behind RMDPs is that, at each time step, the true transition function is chosen adversarially from the uncertainty set \( \uncert \). The dynamics of RMDPs are defined through a pair of policies -- the {\em agent policy} and the {\em environment (or adversarial) policy}. The agent policy $\sigma$ is defined analogously as in the case of MDPs above. On the other hand, the environment policy is a map $\pi: (S\times A)^{\ast} \times S \times A \rightarrow \uncert$ which to each finite history of state-action pairs assigns an element of the uncertainty set $\uncert$. Under the agent policy $\sigma$ and the environment policy $\pi$, the RMDP evolves analogously as an MDP defined above, with the difference that in each time step $t\in \{0,1,2,\dots\}$, the state $s_{t+1}$ is sampled from \(\tran_t(s_t,a_t)\), where \(\tran_t = \pi(s_0,a_0,\ldots,a_{t-1},s_t,a_t)\).

We note that our definition differs a bit from those found in the literature, where the environment policies have limited adaptiveness: they are either stationary \cite{HoPW21} or Markovian (i.e., time-dependent)~\cite{WangVAPZ23}. However, one of the corollaries of our results is that all three definitions lead to equivalent problems in the setting we consider.

We denote by \( \probm^{\sigma,\pi}_{\rmdp} \) and \( \expv^{\sigma,\pi}_{\rmdp} \) the probability measure and the expectation operator induced by the pair of policies \( \sigma, \pi \) over the space of infinite trajectories \(s_0,a_0,s_1,a_1,\ldots \) of \(\rmdp\).

\smallskip\noindent{\bf (s,a)-rectangularity.} In this work, we restrict to the study of \emph{\( (s,a) \)-rectangular uncertainty sets}~\cite{NilimG03,Iyengar05}, which are the most commonly studied class of uncertainty sets and are of the form \( \uncert = \prod_{(s,a)\in S\times A} \uncert_{(s,a)} \), where each \( \uncert_{(s,a)} \subseteq \dist(S) \). The assumption means that the environment can choose transition probabilites for each state-action pair independently from other state-action pairs. 

\smallskip\noindent{\bf Objectives} We consider RMDPs with respect to \emph{long-run average} (also known as \emph{mean payoff}) and {\em discounted-sum} objectives. For each trajectory \( \tau = s_0,a_0,s_1,a_1,\ldots\), we define its long-run average payoff as
\[
\limavg(\tau) = \liminf_{N\rightarrow \infty} \frac{1}{N+1}\cdot \sum_{i=0}^{N}\rew(s_i,a_i).
\]
For every \(\gamma\in(0,1)\), we define its \(\gamma\)-discounted payoff as
\[
\disc_\gamma(\tau) = \lim_{N\rightarrow \infty} \sum_{i=0}^{N}\gamma^i\cdot\rew(s_i,a_i).
\]

\noindent{\bf Problem statement.} The \emph{value} of an RMDP \(\rmdp\) given policies \(\sigma, \pi\) is the expected payoff incurred under the policies $\val(\rmdp,\sigma,\pi) = \expv^{\sigma,\pi}_{\rmdp}[\payoff]$, where \(\payoff\) is either \(\limavg\) or \(\disc_\gamma\).
The \emph{optimal value} of an RMDP \(\rmdp\) is the maximal expected payoff achievable by the agent against any environmental choices of transition functions:
\[\val^*(\rmdp) = \sup_{\sigma} \inf_{\pi} \val(\rmdp,\sigma,\pi).\]
An agent strategy \(\sigma^*\) is \emph{optimal} if it achieves the value against any environment policy, i.e., if $\inf_{\pi}\val(\rmdp,\sigma^*,\pi) = \val^*(\rmdp)$. Given an RMDP with either objective, our goal is to compute its value and a corresponding optimal policy.

\smallskip\noindent{\bf Assumption: polytopic RMDPs.} Towards making our problem computationally tractable, we need to assume a certain form of uncertainty sets in RMDPs. In this work, we consider a very general setting of {\em polytopic uncertainty sets}. An RMDP \(\rmdp =  (S, A, \uncert, \rew, s_0) \) is said to be {\em polytopic} if, for each state-action pair $(s,a) \in S \times A$, the uncertainty set $\uncert(s,a) \subseteq \mathbb{R}^{|S|}$ is a polytope (i.e.~a convex hull of finitely many points in $\mathbb{R}^{|S|}$). These strictly subsume RMDPs with $L^1$-uncertainty (or total variation) sets~\cite{HoPW21}, interval MDPs~\cite{GivanLD00} and contamination models~\cite{WangVAPZ23}. Furthermore, all our results are applicable to uncertainty sets which are {\em subsets of polytopes} so long as all polytope vertices are within the uncertainty set.

\section{Reduction to Turn-based Stochastic Games}
\label{sec:red}

This section presents the main result of this work. We show that the problem of solving long-run average reward polytopic RMDPs admits a linear-time reduction to the problem of solving long-run average reward turn-based stochastic games. The reduction is linear in the size of the RMDP and the number of vertices of uncertainty polytopes. In what follows, we first recall turn-based stochastic games. We then formally define and prove our reduction, which is the main theoretical result of this work. Finally, we use the reduction to obtain a sequel of novel results on RMDPs by leveraging results on turn-based stochastic games.

\subsection{Background on Turn-based Stochastic Games}

Turn-based stochastic games are a standard model of decision making in the presence of both adversarial agent and stochastic uncertainty~\cite{Shapley:stochastic-games,FV:book,Kucera:games-chapter}. 
Formally, a (finite-state, two-player, zero-sum) \emph{turn-based stochastic game (TBSG)} is a tuple \( \mathcal{G} =  (S,A,\delta,\rew, s_0) \) where all the elements are defined as in MDPs, with the additional requirement that the state set \( S \) is partitioned into sets \( S_{\max} \) and \( S_{\min}\) and the action set \( A \) is partitioned into sets \( A_{\max} \) and \( A_{\min}\). We say that these vertices and actions belong to players \emph{Max} and \emph{Min}, respectively.
Similarly to MDPs and RMDPs, the dynamics of TBSGs are defined through the notion of policies for each player. We distinguish between Max- and Min-histories depending on whether the history ends in a state belonging to Max or Min. 

The value of a policy pair, optimal value and optimal policy for long-run average and discounted-sum objectives are defined analogously as for RMDPs. We denote by $\Sigma$ and $\Pi$ the sets of all policies of Max and Min, respectively. Similarly, we denote by $\Sigma_p$ and $\Pi_p$ the sets of all pure positional policies of Max and Min, respectively. The following is a classical result on TBSGs that will be an important ingredient in constructing our reduction from polytopic RMDPs. Intuitively, it says that both players can achieve the optimal expected payoff by using pure positional policies. For long-run average objective, the result was proved by~\cite{Gillette:mean-payoff,LL:mean-payoff-SIAM}. For discounted-sum, positional determinacy was proved by~\cite{Shapley:stochastic-games}, pure positional determinacy is a simple consequence~\cite{stochgameschapter}.

\begin{theorem}[Pure positional determinacy] \label{thm:determinacy}
    Given a TBSG $\mathcal{G}$, the following equality holds for both long-run average and discounted-sum objectives:
    \begin{equation*}
    \begin{split}
        \val^* &(\mathcal{G}) = \sup_{\sigma \in \Sigma} \inf_{\pi \in \Pi} \val(\mathcal{G},\sigma,\pi) = \inf_{\pi \in \Pi} \sup_{\sigma \in \Sigma} \val(\mathcal{G},\sigma,\pi) \\
        &= \max_{\sigma \in \Sigma_p} \min_{\pi \in \Pi_p} \val(\mathcal{G},\sigma,\pi) = \min_{\pi \in \Pi_p} \max_{\sigma \in \Sigma_p} \val(\mathcal{G},\sigma,\pi).
    \end{split}
    \end{equation*}
\end{theorem}

\noindent{\bf Comparison to RMDPs.} While the adversarial structure of TBSGs mimics the one of RMDPs, there is a crucial distinction: in TBSGs, both players may only select from {\em finitely many actions} at each step, while in RMDPs, the adversary can choose from a possibly {\em continuous set of distributions}. Hence, RMDPs do not immediately reduce to TBSGs.



\subsection{Reduction}
\label{subsec:reduction}



\noindent{\bf Intuition.} A polytopic RMDP $\rmdp$ induces a TBSG $\mathcal{G}_\rmdp$ as follows. The TBSG $\mathcal{G}_\rmdp$ is played between the agent and the environment players, which alternate in turns. The goal of the agent is to maximize long-run average reward, whereas the goal of the environment is to minimize it. Hence, the agent is the Max player and the environment is the Min player.

The agent and the environment players alternate in turns. First, the TBSG is in a state belonging to the agent player and the agent player selects an action to play. The state and the action sets belonging to the agent player in the TBSG coincide with the state and the action sets of the agent in the RMDP. Thus, the TBSG dynamics in states belonging to the agent player simulate the RMDP dynamics. On the other hand, the state and the action sets belonging to the environment player in the TBSG coincide with the set of all state-action pairs and the set of all vertices of uncertainty polytopes in the RMDP, respectively. Thus, the TBSG dynamics in states belonging to the environment player simulate the process in which a probabilistic transition function from which the next state will be sampled is selected from the uncertainty polytope in the RMDP. Note that pure strategies of the environment player can only select probabilistic transition functions that are vertices of the uncertainty polytopes, however any other element of the uncertainty polytope can be selected by using randomized strategies. Finally, the reward function of the TBSG also simulates that of the RMDP -- the rewards of state-action pairs belonging to the adversary player in the TBSG are twice the rewards in the RMDP, whereas the rewards induced by the environment player's moves in the TBSG are $0$. Thus, under every policy pair, the TBSG intuitively induces an infinite play whose long-run average is the same as that in the RMDP. In what follows, we formalize this construction.



\smallskip\noindent{\bf Formal definition.} Let \( \rmdp = (S^\rmdp, A^\rmdp, \uncert^\rmdp, \rew^\rmdp, s_0^\rmdp) \) be a polytopic RMDP. We define a TBSG \( \mathcal{G}_\rmdp =  (S^{\mathcal{G}_\rmdp},A^{\mathcal{G}_\rmdp},\delta^{\mathcal{G}_\rmdp},\rew^{\mathcal{G}_\rmdp}, s_0^{\mathcal{G}_\rmdp}) \) as below, and say that $\mathcal{G}_\rmdp$ is {\em induced} by $\rmdp$. In what follows, for each state-action pair $(s,a) \in S^\rmdp \times A^\rmdp$ in the RMDP $\rmdp$, we use $V^\rmdp_{s,a}$ to denote the set of all vertices of the uncertainty polytope $\uncert^\rmdp(s,a)$:
\begin{compactitem}
    \item {\em States.} $S^{\mathcal{G}_\rmdp}$ is defined as a union of two disjoint sets $S^{\mathcal{G}_\rmdp}_{\max}$ and $S^{\mathcal{G}_\rmdp}_{\min}$, where $S^{\mathcal{G}_\rmdp}_{\max} = S^\rmdp$ and $S^{\mathcal{G}_\rmdp}_{\min} = S^\rmdp \times A^\rmdp$
    are copies of the sets of all states and all state-action pairs in $\rmdp$.
    \item {\em Actions.} $A^{\mathcal{G}_\rmdp}$ is defined as a union of two disjoint sets $A^{\mathcal{G}_\rmdp}_{\max}$ and $A^{\mathcal{G}_\rmdp}_{\min}$, where
    $A^{\mathcal{G}_\rmdp}_{\max} = A^\rmdp$ and $A^{\mathcal{G}_\rmdp}_{\min} = \cup_{(s,a) \in S^\rmdp \times A^\rmdp} V^\rmdp_{s,a}$
    are copies of the sets of all actions and all vertices of uncertainty polytopes in $\rmdp$.
    \item {\em Transition function.} $\delta^{\mathcal{G}_\rmdp}: S^{\mathcal{G}_\rmdp}_{\max} \times A^{\mathcal{G}_\rmdp}_{\max} \cup S^{\mathcal{G}_\rmdp}_{\min} \times A^{\mathcal{G}_\rmdp}_{\min} \rightarrow \Delta(S^{\mathcal{G}_\rmdp})$ is defined as follows:
    \begin{compactitem}
        \item For state-action pairs $(s,a) \in S^{\mathcal{G}_\rmdp}_{\max} \times A^{\mathcal{G}_\rmdp}_{\max}$ belonging to Max, we let $\delta^{\mathcal{G}_\rmdp}(s,a)$ be the Dirac probability distribution over $\Delta(S^{\mathcal{G}_\rmdp})$ with all probability mass in the Min state $(s,a) \in S^{\mathcal{G}_\rmdp}_{\min}$, i.e.
        \begin{equation*}
        \delta^{\mathcal{G}_\rmdp}(s,a)(s') = \begin{cases}
            1, &\text{if } s' = (s,a) \in S^{\mathcal{G}_\rmdp}_{\min}, \\
            0, &\text{otherwise.}
        \end{cases}
        \end{equation*}
        \item For state-action pairs $((s,a),v) \in S^{\mathcal{G}_\rmdp}_{\min} \times A^{\mathcal{G}_\rmdp}_{\min}$ belonging to Min, we let $\delta^{\mathcal{G}_\rmdp}(s,a)$ be the probability distribution over $S^{\mathcal{G}_\rmdp}_{\max} = S^\rmdp$ induced by vertex $v$ of uncertainty polytope $\uncert^\rmdp(s,a)$, i.e.
        \begin{equation*}
        \delta^{\mathcal{G}_\rmdp}((s,a),v)(s') = \begin{cases}
            v(s'), &\text{if } s' \in S^{\mathcal{G}_\rmdp}_{\max}, \\
            0, &\text{otherwise.}
        \end{cases} 
        \end{equation*}
    \end{compactitem}
    \item {\em Reward function.} $r^{\mathcal{G}_\rmdp}: S^{\mathcal{G}_\rmdp}_{\max} \times A^{\mathcal{G}_\rmdp}_{\max} \cup S^{\mathcal{G}_\rmdp}_{\min} \times A^{\mathcal{G}_\rmdp}_{\min} \rightarrow \mathbb{R}$ is defined as follows:
    \begin{compactitem}
        \item For state-action pairs $(s,a) \in S^{\mathcal{G}_\rmdp}_{\max} \times A^{\mathcal{G}_\rmdp}_{\max}$, the reward is twice the reward incurred in the RMDP $\rmdp$, i.e.~$
        r^{\mathcal{G}_\rmdp}(s,a) = 2 \cdot r^{\rmdp}(s,a)$.
        \item For state-action pairs $((s,a),v) \in S^{\mathcal{G}_\rmdp}_{\min} \times A^{\mathcal{G}_\rmdp}_{\min}$, the reward is $0$, i.e.~$r^{\mathcal{G}_\rmdp}((s,a),v) = 0$.
    \end{compactitem}
    \item {\em Initial state.} $s_0^{\mathcal{G}_\rmdp} = s_0^{\rmdp} \in S^{\mathcal{G}_\rmdp}_{\max}$, i.e. the initial state in $\mathcal{G}_\rmdp$ is the copy of the initial state in $\rmdp$.
\end{compactitem}
The following theorem establishes reduction correctness. In the proof, considering polytopic RMDPs allows us to soundly prune the action space of the environment up to corners of the polytope. This is then used to reduce the RMDP to a TBSG with finite state and action space.

\begin{theorem}[Correctness]\label{thm:soundness}
    Consider a polytopic RMDP \( \rmdp = (S^\rmdp, A^\rmdp, \uncert^\rmdp, \rew^\rmdp, s_0^\rmdp) \) and define a TBSG \( \mathcal{G}_\rmdp =  (S^{\mathcal{G}_\rmdp},A^{\mathcal{G}_\rmdp},\delta^{\mathcal{G}_\rmdp},\rew^{\mathcal{G}_\rmdp}, s_0^{\mathcal{G}_\rmdp}) \) as above. Then, for long-run average objective we have $\val^* (\mathcal{\rmdp}) = \val^* (\mathcal{G}_\rmdp)$.
\end{theorem}

\begin{proof}[Proof sketch, full proof in Appendix~\ref{app:soundnessproof}]
    Denote by $\Sigma^{\rmdp}$ and $\Pi^{\rmdp}$ the sets of all policies of the agent and the environment in the polytopic RMDP, and similarly define $\Sigma^{\mathcal{G}_\rmdp}$ and $\Pi^{\mathcal{G}_\rmdp}$ in the TBSG. We use $\Sigma^{\rmdp}_p$, $\Pi^{\rmdp}_p$, $\Sigma^{\mathcal{G}_\rmdp}_p$ and $\Pi^{\mathcal{G}_\rmdp}_p$ to denote the subsets of pure positional policies. In the full proof, we first define mappings from the sets of agent and environment policies in the polytopic RMDP to the sets of agent player and environment player policies in the TBSG:
    \begin{equation*}
        \Phi: \Sigma^{\rmdp} \rightarrow \Sigma^{\mathcal{G}_\rmdp}, \hspace{1em}
        \Psi: \Pi^{\rmdp} \rightarrow \Pi^{\mathcal{G}_\rmdp},
    \end{equation*}
    We then show that these mappings preserve the values of policy pairs under long-run average objective, i.e.~that
    \begin{equation*}
        \val(\rmdp,\sigma,\pi) = \val(\mathcal{G}_\rmdp,\Phi(\sigma),\Psi(\pi))
    \end{equation*}
    hold for each $\sigma \in \Sigma^{\rmdp}$ and $\pi \in \Pi^{\rmdp}$. Finally, we show that when restricted to pure positional policies, these mappings give rise to {\em one-to-one correspondences} (i.e.~{\em bijections})
    \begin{equation*}
        \Phi_p: \Sigma^{\rmdp}_p \rightarrow \Sigma^{\mathcal{G}_\rmdp}_p, \hspace{1em}
        \Psi_p: \Pi^{\rmdp}_p \rightarrow \Pi^{\mathcal{G}_\rmdp}_p.
    \end{equation*}
    Since by Theorem~\ref{thm:determinacy} there exist pure positional policies that achieve the optimal value for each player in the TBSG, we can conclude that there exist pure positional policies that achieve the optimal value in the polytopic RMDP and that the optimal values in the polytopic RMDP and the TBSG coincide.
\end{proof}


In order to study the computational complexity of the reduction, we first need to fix a representation of polytopic RMDPs. In particular, we assume that each polytopic RMDP is represented as an ordered tuple \( \rmdp = (S^\rmdp, A^\rmdp, \uncert^\rmdp, \rew^\rmdp, s_0^\rmdp) \), where $S^\rmdp$ is a list of states, $A^\rmdp$ is a list of actions, $\uncert^\rmdp$ is a list of polytope vertices for each state-action pair, $\rew^\rmdp$ is a list of rewards for each state-action pair and $s_0^\rmdp$ is an element of the state list. The following theorem shows that the above is a linear-time reduction in the size of the RMDP and the number of vertices of uncertainty polytopes.

\begin{theorem}[Complexity, proof in Appendix~\ref{app:complexityproof}]\label{thm:complexity}
    Consider a polytopic RMDP \( \rmdp = (S^\rmdp, A^\rmdp, \uncert^\rmdp, \rew^\rmdp, s_0^\rmdp) \) and define a TBSG \( \mathcal{G}_\rmdp =  (S^{\mathcal{G}_\rmdp},A^{\mathcal{G}_\rmdp},\delta^{\mathcal{G}_\rmdp},\rew^{\mathcal{G}_\rmdp}, s_0^{\mathcal{G}_\rmdp}) \) as above. Then, the size of $\mathcal{G}_\rmdp$ is linear in the size of $\rmdp$ and the number of vertices of uncertainty polytopes.
\end{theorem}


\subsection{Discussion and Implications}

Our reduction enables us to leverage the rich repository of results about TBSGs achieved within the computational game theory and formal methods communities. Since our reduction yields a well-defined and effective correspondence between policies in the polytopic RMDP and the TBSG, the results for TBSGs naturally carry over to the polytopic RMDP setting. In this subsection we highlight several such results, focusing on the issues of complexity and efficient algorithms.

\smallskip\noindent{\bf Pure positional determinacy.} First, as an immediate corollary of our reduction and of Theorem~\ref{thm:determinacy}, we get the following result. While this result was already established for polytopic RMDPs by a direct analysis and without the reduction to TBSGs~\cite{GoyalG23}, we state it below as it will be used for proving subsequent results.

\begin{cor}
\label{cor:rmdp-determinacy}
In a polytopic RMDP under long-run average objective, both the agent and the environment have pure positional optimal policies. 
\end{cor}

\noindent{\bf Problem complexity.} The decision problems associated with solving long-run average reward TBSGs (such as deciding whether the optimal value of the game is at least a given threshold \(v\)) are known to be in \(\textsc{NP} \cap \textsc{coNP}\). This is due to Theorem~\ref{thm:determinacy}: for \textsc{NP} membership, one can guess an optimal pure positional policy of, say player Max, and then verify that it enforces value \(\geq v\): the verification can be done in polynomial time since fixing a pure positional policy in an RMDP yields a standard MDP that can be solved in polynomial time e.g. by linear programming~\cite{Puterman94}. The argument for \textsc{coNP} membership is dual. By Corollary~\ref{cor:rmdp-determinacy}, the same characterization holds for RMDPs.

\begin{cor}
\label{cor:complexity}
Given an RMDP \(\rmdp\) under long-run average objective and a threshold \(v \in \mathbb{Q}\), the problem whether \(\val^*(\rmdp) \geq v\) belongs to \(\textsc{NP} \cap \textsc{coNP}\).
\end{cor}

\noindent{\bf Efficient algorithms.} \emph{Simple stochastic games (SSGs)} \cite{Condon:simple-stochastic-games-IC} are a special class of TBSGs with reachability objectives that captures the complexity of the whole TBSG class. It was shown by~\cite{AndersonMiltersen:Stoch-Games-complexity} that solving TBSGs with both discounted and long-run average payoffs is polynomial-time reducible to solving SSGs (where by solving we mean computing optimal values and policies). 

While SSGs are not known to be solvable in polynomial time, several sub-exponential algorithms for them were developed. The randomized algorithm due to~\cite{Ludwig:SSG-randomized-IC} solves a given SSG in time \(\mathcal{O}(2^{\sqrt{\min\{|S_{max}|,|S_{min}|\}}}\cdot \mathit{poly{(|\mathcal{G}|)}})\), where \(|\mathcal{G}|\) is the encoding size of the game. Neither our reduction from RMDPs to TBSGs nor the reduction from TBSGs to SSGs~\cite{AndersonMiltersen:Stoch-Games-complexity} add any player Max vertices. Hence, we get the following result.

\begin{cor}
\label{cor:subexp}
There is a randomized algorithm for solving long-run average reward polytopic RMDPs  whose expected runtime is \({\mathcal{O}(2^{\sqrt{|S|}}\cdot \mathit{poly{(|\rmdp|)}})}\), i.e., sub-exponential.
\end{cor}

\begin{remark}[Discounted polytopic RMDPs]
    In Appendix~\ref{app:discounted}, we show the above reduction can be modified to also work for discounted-sum objectives. Hence Corollaries~\ref{cor:rmdp-determinacy}--\ref{cor:subexp} hold for discounted-sum objectives as well.
\end{remark}

In \cite{Chatterjee23} authors provide an algorithm for solving TBSGs parameterized by the treewidth of the underlying graph. The algorithm has runtime $\mathcal{O}((t n^2 \log nW)^{t \log n})$ where $n$ is the number of nodes, $t$ is the treewidth of the graph and $W$ is the maximum absolute value of rewards in the game. Let $\rmdp$ be an RMDP with $m$ actions and whose underlying graph has treewidth $t$. In Appendix~\ref{app:treewidth}, we show that $\mathcal{G}_\rmdp$ has treewidth at most $(t+1)(m+1)-1$, which implies the following result.

\begin{cor}
    Given an RMDP $\rmdp$ with constant number of actions and treewidth of the underlying graph, the long-run average objective can be solved in quasi-polynomial time. 
\end{cor}




\section{Algorithm for Long-run Average RMDPs} \label{sec:algo}

In this section, we present our algorithm \emph{Robust Polytopic Policy Iteration (RPPI)} for computing long-run average values and optimal pure positional policies in polytopic RMDPs. The pseudocode of our algorithm is shown in Algorithm~\ref{alg:RPPI}. Our algorithm can be viewed as a policy iteration algorithm for solving long-run average reward polytopic RMDPs. Policy iteration was hitherto {\em not known to be sound} for solving long-run average reward RMDPs, thus the design of a policy iteration algorithm is another contribution of our work.



\smallskip\noindent{\bf Outline.} RPPI uses the reduction in Section~\ref{sec:red} to reduce our problem to solving long-run average TBSGs, for which policy iteration is known to be sound and highly efficient~\cite{HansenMZ13}. Our RPPI is thus motivated by efficient implementations of policy iteration for TBSGs, which utilize Blackwell optimality and use policy iteration to solve a number of discounted-sum TBSGs while increasing the discount factor until a discount factor for which discounted-sum and long-run average policies are equal is achieved.

\smallskip\noindent{\bf Algorithm details.} 
Algorithm~\ref{alg:RPPI} takes as input a polytopic RMDP $\rmdp$ (line~1). It then constructs a TBSG $\mathcal{G}_\rmdp$ as in Section~\ref{sec:red} to reduce our problem to solving a long-run average TBSG (line~2), initializes a discount factor $\gamma=\frac{1}{2}$ (line~3) and performs policy iteration for long-run average TBSGs (lines 4-10). For each value of the discount factor $\gamma$, Algorithm~\ref{alg:RPPI} first solves the discounted-sum TBSG and computes a pair of optimal policies by using the discounted-sum TBSG policy iteration algorithm of~\cite{HansenMZ13}, which runs in polynomial time, as an off-the-shelf subprocedure (line~5). It then uses Policy Profile Evaluation (PPE, see details below) to check if this policy profile is optimal for the $\limavg$ objective as well (line~6). If the policy profile is found to be optimal, then the algorithm breaks the loop (lines~7-8) 
and returns the agent policy together with its $\limavg$ value (line~12-13). Otherwise, the discount factor $\gamma$ is increased (line~10) and Algorithm~\ref{alg:RPPI} proceeds to solving the next discounted-sum TBSG (line~4-10). By Blackwell optimality for TBSGs, Algorithm~\ref{alg:RPPI} is guaranteed to eventually terminate and find a pair of optimal pure positional policies for each player in the TBSG $\mathcal{G}_\rmdp$. By the one-to-one correspondence between pure positional policies in the RMDP $\rmdp$ and the TBSG $\mathcal{G}_\rmdp$ established in the proof of Theorem~\ref{thm:soundness} and in Corollary~\ref{cor:rmdp-determinacy}, the agent player policy then gives rise to an optimal pure positional agent policy in the RMDP $\rmdp$, as desired.

\smallskip\noindent{\bf Policy profile evaluation.} We now describe the PPE subprocedure used for policy profile evaluation in line~6 of Algorithm~\ref{alg:RPPI}. 
PPE takes as input the TBSG $\mathcal{G}_\rmdp$ and policies $\sigma$ and $\pi$ of Max and Min players. In order to check whether $\sigma$ and $\pi$ are optimal for the long-run average objective, by pure positional determinacy of TBSGs (Theorem~\ref{thm:determinacy}) it suffices to checks if they provide optimal responses to each other.

To check if $\sigma$ and $\pi$ provide optimal responses to each other, PPE first considers an MDP obtained by fixing policy $\pi$ of the Min player and solves the long-run average MDP to compute the optimal long-run average payoff of Max player against $\pi$. Similarly, PPE then considers the MDP obtained by fixing policy $\sigma$ of the Max player and solves the long-run average MDP to compute the optimal long-run average payoff of Min player against $\sigma$. In both cases, the long-run average MDP can be solved via linear programming or policy iteration \cite{Puterman94}. If two MDPs have equal long-run average payoffs, then we conclude that two policies are optimal responses to each other, otherwise they are not.


The following theorem establishes the correctness of our RPPI algorithm (and the PPE subprocedure). Its proof follows from the correctness of our reduction (Theorem~\ref{thm:soundness}), Blackwell optimality for TBSGs~\cite{AndersonMiltersen:Stoch-Games-complexity}, and the correctness of all involved subprocedures for which we use off-the-shelf approaches.

\begin{theorem}
    Given a polytopic RMDP $\rmdp$, Algorithm~\ref{alg:RPPI} returns the optimal long-run average value $\val^*(\rmdp)$ and an optimal policy for the agent.
\end{theorem}

\begin{algorithm}[t]
\caption{Robust Polytopic Policy Iteration (RPPI)} 
\label{alg:RPPI}
\begin{algorithmic}[1]
\STATE \textbf{Input} polytopic RMDP $\rmdp$
\STATE $\mathcal{G}_\rmdp$ $\leftarrow$ TBSG by reduction from $\rmdp$ as in Section~\ref{sec:red}
\STATE $\gamma$ $\leftarrow$ $\frac{1}{2}$ \COMMENT{initial discount factor}
\WHILE{True}
\STATE $\sigma$, $\pi$ $\leftarrow$ optimal discounted-sum policies in $\mathcal{G}_\rmdp$ with discount factor $\gamma$, by method of~\cite{HansenMZ13}
\STATE $is\_optimal, rewards$ $\leftarrow$ PPE$(\mathcal{G}_\rmdp, \sigma, \pi)$
\IF{$is\_optimal$}
\STATE break
\ENDIF
\STATE $\gamma$ $\leftarrow$ $\frac{1+\gamma}{2}$
\ENDWHILE
\STATE $\val^*(\rmdp) = rewards[s_0]$
\STATE \textbf{Return} $\val^*(\rmdp), \sigma$
\end{algorithmic}
\end{algorithm}



\section{Experimental Results}

We implement RPPI (Algorithm~\ref{alg:RPPI}) and compare it against two state of the art value iteration-based methods for solving long-run average reward RMDPs with uncertainty sets not being intervals or $L^1$-balls, towards demonstrating the significant computational runtime gains provided by a policy iteration-based algorithm. Furthermore, we demonstrate the applicability of our method to non-unichain polytopic RMDPs to which existing algorithms are not applicable. Our implementation is publicly available at \url{https://github.com/mehrdad76/RMDP-LRA}.

\smallskip\noindent{\bf Implementation details.} We implement the policy iteration for discounted-sum TBSGs of~\cite{HansenMZ13} to compute optimal strategies in line 5 in Algorithm \ref{alg:RPPI}. Furthermore, we use the probabilistic model checker Storm~\cite{STORM} as an off-the-shelf tool to compute values in long-run average reward MDPs in the Policy Profile Evaluation subprocedure. Due to numerical precision issues in Python and Storm, in PPE we replace the equality by $|max\_rewards-min\_rewards|<10^{-5}$. All experiments were run in Python 3.9 on a Ubuntu 22.04 machine with an octa-core 2.40 GHz Intel Core i5 CPU, 16 GB RAM. 

\smallskip\noindent{\bf Baselines.} We compare the runtime of our RPPI against two value iteration-based algorithms: Robust Value Iteration (RVI) and Robust Relative Value Iteration (RRVI), both proposed in~\cite{WangVAPZ23}. Similarly to our RPPI, RVI uses Blackwell optimality to compute the long-run average value and optimal policy by solving discounted-sum reward RMDPs and gradually increasing the discount factor. On the other hand, RRVI is based on a value iteration procedure that directly solves the long-run average reward RMDP. 

Both RVI and RRVI impose {\em structural restrictions} on the RMDP. RVI requires the RMDP to be a unichain. RRVI requires the RMDP to be unichain and aperiodic. In contrast, our RPPI does not impose any of these assumptions.

Moreover, while RVI and RRVI are guaranteed to converge in the limit, they do not provide a stopping criteria for practical implementation. Hence, in our comparison, we run them both until the absolute difference between their computed values and our computed value is at most $10^{-3}$. 

\smallskip\noindent{\bf Benchmarks.} We consider the following benchmarks:

\smallskip\textbf{1. Contamination Model.} Contamination Model is a conservative model of an RMDP defined over a given MDP where, at each step, the probability of taking the presrcibed MDP transition function is $R \in [0,1$] whereas with probability $1-R$ the environment gets to choose any probability distribution over the states. In our evaluation, we consider a Contamination Model taken from~\cite{WangVAPZ23}. For a given natural number $n$, the base MDP contains $n$ states and $n+10$ actions. The transition function and rewards for each state-action pair are then randomly generated (see~\cite{WangVAPZ23} for details). The uncertainty set $\uncert_{(s,a)}$ for state-action pair $(s,a)$ is $\uncert_{(s,a)} = \{ (1-R)\tran_{(s,a)}+R\cdot p ~|~ p \in \dist(S) \}$. The resulting RMDP is polytopic, unichain and aperiodic, hence both baselines are applicable to it.

\smallskip\textbf{2. Frozen Lake.} This benchmark modifies the Frozen Lake environment in the OpenAI Gym \cite{OpenAIGym} to turn it into an RMDP. Consider an $n \times n$ grid with some cells being holes. An agent starts at the top left corner and chooses an action from $\{\textit{left},\textit{right},\textit{up},\textit{down}\}$ to move to the next cell in that direction. However, the grid is slippery and the agent might move in a direction perpendicular to the chosen direction. If the agent falls into a hole, they are stuck and remain there with probability~$1$. The reward increases as the agent gets closer to the bottom right corner of the grid, so the agent should get to and stay around the bottom right corner. We turn this model into a polytopic RMDP by allowing the adversarial environment to perturb the transition probabilities by increasing the probability of moving to one adjacent cell by at most $d = 0.2$ while decreasing the remaining probabilities uniformly. Note that this RMDP is aperiodic and multichain, hence neither RVI nor RRVI baselines are applicable. We use this environment to evaluate our RPPI on multichain polytopic RMDPs to which existing methods are not applicable. In order to allow for a comparison, we also consider a unichain variant of the Frozen Lake environment where the holes are treated as ''walls'' and cannot be visited. To distinguish between the two variants, we refer to them as {\em Unichain} and {\em Multichain} Frozen Lake models, respectively.


\smallskip\noindent{\bf Results on Contamination Model.} We compared our method against both baselines and our results are shown in Figure~\ref{fig:contamination}. As indicated by the results, our RPPI is considerably faster than RVI. Given the similar structure of RPPI and RVI with the main difference being the use policy instead of value iteration, our results confirm the significant gains in computational runtime provided by our reduction to TBSGs which allowed the design of a policy iteration algorithm. 
On the other hand, we observe that the RRVI baseline is faster than both our method and the RVI baseline. The explanation for this is that RRVI is an algorithm specifically tailored to unichain aperiodic RMDPs, whereas our RPPI imposes no structural restrictions on the RMDP.

\begin{figure}[t]
	\centering
        \vspace{-1em}
	\includegraphics[scale=0.28]{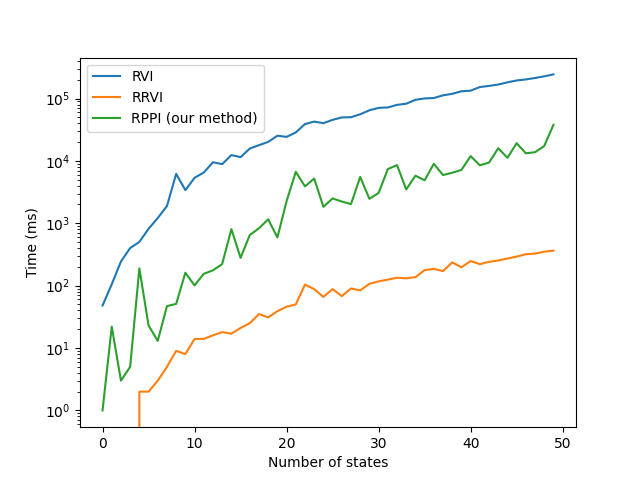}
	\caption{Runtime comparison on the Contamination Model. 
    }
	\label{fig:contamination}
\end{figure}

\smallskip\noindent{\bf Results on Frozen Lake.} Table~\ref{table:frozen} shows our results on the Frozen Lake model. As mentioned before, due to RMDP structural restrictions imposed by the baselines, RVI is only applicable to the Unichain model whereas RRVI is applicable to neither of the models. Our RPPI imposes no structural restrictions on the RMDP and is thus applicable to both models.
As we can see from Table~\ref{table:frozen} our policy iteration-based RPPI is significantly faster than RVI, with an average runtime of 7.8 seconds. In contrast, the RVI method takes more than 2000 seconds on average, and times out on the $10 \times 10$ model. Moreover, our method successfully solves all the Multichain model instances in less than 4 minutes each. 

\begin{table}[t]
\caption{Runtime comparison on the Frozen Lake models.}
\label{table:frozen}
\begin{tabular}{|c|c|c|c|}
	\hline
	& \multicolumn{2}{c|}{Unichain} & Multichain \\
	\hline
	n & RPPI Time (s) & VI Time (s) & RPPI Time (s) \\
	\hline
	\hline
	2 & {\bf 0.0} & 0.15 & {\bf 0.0} \\
	\hline
	3 & {\bf 0.0} & 8.3 & {\bf 2.6} \\
	\hline
	4 & {\bf 0.0} & 156.8 & {\bf 0.6} \\
	\hline
	5 & {\bf 0.1} & 106.6 & {\bf 4.0} \\
	\hline
	6 & {\bf 6.1} & 1232.6 & {\bf 53.7} \\
	\hline
	7 & {\bf 17.4} & 3347.8 & {\bf 119.0} \\
	\hline
	8 & {\bf 3.7} & 7383.6 & {\bf 104.8} \\
	\hline
	9 & {\bf 19.2} & 6496.6 & {\bf 93.0} \\
	\hline
	10 & {\bf 23.7}  & Timeout & {\bf 231.4} \\
	\hline
	\hline 
	Average & {\bf 7.8} & 2341.5 & {\bf 67.6}\\
	STDEV	& 9.5 & 3061.8 & 78.2\\
	\hline
\end{tabular}
\end{table}

\smallskip\noindent{\bf Concluding remarks.} We considered the problem of solving long-run average reward RMDPs with polytopic uncertainty sets and proposed a new perspective on this problem by establishing a linear-time reduction to long-run average reward turn-based stochastic games. This reduction allowed us to leverage results from the stochastic games literature and obtain a number of results on computational complexity and efficient algorithms that were hitherto not known to hold for long-run average polytopic RMDPs. 
Our work leaves several interesting venues for future work. 
A natural question to consider is whether and how RMDPs with non-rectangular uncertainty sets could be related to stochastic games. Another direction would be to consider the connection between RMDPs with non-polytopic uncertainty sets and stochastic games, e.g.~when uncertainty sets are specified in terms of KL-divergence from some nominal distribution~\cite{WangVAPZ23}. The extension to non-polytopic uncertainty sets via either (a) approximating non-polytopic with polytopic uncertainty sets and applying our method, or (b) designing new approaches, are interesting
directions of future work.

\section*{Acknowledgements}

This work was supported in part by the ERC-2020-CoG 863818 (FoRM-SMArt) and the Czech Science Foundation grant no. GA23-06963S

\bibliographystyle{named}
\bibliography{bibliography}

\newpage
\appendix

\begin{center}
\Large
Appendix
\end{center}

\section{Proof of Theorem~2}\label{app:soundnessproof}

\begin{theorem*}[Correctness]
    Consider a polytopic RMDP \( \rmdp = (S^\rmdp, A^\rmdp, \uncert^\rmdp, \rew^\rmdp, s_0^\rmdp) \) and define a TBSG \( \mathcal{G}^\rmdp =  (S^{\mathcal{G}_\rmdp},A^{\mathcal{G}_\rmdp},\delta^{\mathcal{G}_\rmdp},\rew^{\mathcal{G}_\rmdp}, s_0^{\mathcal{G}_\rmdp}) \) as above. Then, for long-run average objective we have
    \[ \val^* (\mathcal{\rmdp}) = \val^* (\mathcal{G}^\rmdp). \]
\end{theorem*}

\begin{proof}
    Let $\Sigma^{\rmdp}$ and $\Pi^{\rmdp}$ denote the sets of all policies of the agent and the environment in the RMDP, and let $\Sigma^{\mathcal{G}_\rmdp}$ and $\Pi^{\mathcal{G}_\rmdp}$ denote the sets of all policies of the agent player and the environment player in the TBSG. We use $\Sigma^{\rmdp}_p$, $\Pi^{\rmdp}_p$, $\Sigma^{\mathcal{G}_\rmdp}_p$ and $\Pi^{\mathcal{G}_\rmdp}_p$ to denote the subsets of pure positional policies in these sets. Our proof proceeds in four steps. First, we define the mappings
    \begin{equation*}
        \Phi: \Sigma^{\rmdp} \rightarrow \Sigma^{\mathcal{G}_\rmdp}, \hspace{1em}
        \Psi: \Pi^{\rmdp} \rightarrow \Pi^{\mathcal{G}_\rmdp}.
    \end{equation*}
    Second, we show that these mappings preserve the values of policy pairs. Third, we show that these mappings give rise to one-to-one correspondences (i.e.~bijections) between pure positional policies in RMDPs and TBSGs
    \begin{equation*}
        \Phi_p: \Sigma^{\rmdp}_p \rightarrow \Sigma^{\mathcal{G}_\rmdp}_p, \hspace{1em}
        \Psi_p: \Pi^{\rmdp}_p \rightarrow \Pi^{\mathcal{G}_\rmdp}_p.
    \end{equation*}
    Fourth, we use all the above ingredients to conclude the theorem claim and show that $\val^* (\mathcal{\rmdp}) = \val^* (\mathcal{G}^\rmdp)$.

    \paragraph{Step~1: Definition of $\Phi$} Consider an RMDP agent policy
    \[ \sigma: (S^\rmdp\times A^\rmdp)^{\ast} \times S^\rmdp \rightarrow \dist(A^\rmdp). \]
    We define a TBSG agent player policy
    \[ \Phi(\sigma): (S_{\max}^{\mathcal{G}_\rmdp} \times A_{\max}^{\mathcal{G}_\rmdp} \cup S_{\min}^{\mathcal{G}_\rmdp} \times A_{\min}^{\mathcal{G}_\rmdp})^{\ast} \times S_{\max}^{\mathcal{G}_\rmdp} \rightarrow \Delta(A_{\max}^{\mathcal{G}_\rmdp}). \]
    as follows. In order to define $\Phi(\sigma)$, we need to specify $\Phi(\sigma)(h)$ for every Max-history $h$ in the TBSG $\mathcal{G}_\rmdp$. Let $h$ be a Max-history in $\mathcal{G}_\rmdp$. Since the adversary and the environment players alternate in turns in the TBSG $\mathcal{G}_\rmdp$, the history $h$ must be of form $h  = s_0,a_0,s_1,a_1,\dots,s_{2k-2},a_{2k-2},s_{2k-1},a_{2k-1},s_{2k}$ for some $k \in \mathbb{N}_0$, where even-indexed states and actions belong to the adversary player and odd-indexed states and actions belong to the environment player. Let $h' = s_0,a_0,\dots,s_{2k-2},a_{2k-2},s_{2k}$ be the subsequence consisting only of states and actions belonging to the adversary player in $\mathcal{G}_\rmdp$. Then $h'$ defines a history in the RMDP, and we define
    \[ \Phi(\sigma)(h) := \sigma(h') \]
    where $\sigma(h')$ is interpreted as a probability distribution over $A_{\max}^{\mathcal{G}_\rmdp}$ which recall is defined as a copy of $A^\rmdp$.

    \paragraph{Step~1: Definition of $\Psi$} Consider an RMDP environment policy
    \[ \pi: (S^\rmdp\times A^\rmdp)^{\ast} \times S^\rmdp \times A^\rmdp \rightarrow \uncert^\rmdp. \]
    We define a TBSG environment player policy
    \[ \Psi(\pi): (S_{\max}^{\mathcal{G}_\rmdp} \times A_{\max}^{\mathcal{G}_\rmdp} \cup S_{\min}^{\mathcal{G}_\rmdp} \times A_{\min}^{\mathcal{G}_\rmdp})^{\ast} \times S_{\min}^{\mathcal{G}_\rmdp} \rightarrow \Delta(A_{\min}^{\mathcal{G}_\rmdp}). \]
    as follows. In order to define $\Psi(\pi)$, we need to specify $\Psi(\pi)(h)$ for every Min-history $h$ in the TBSG $\mathcal{G}_\rmdp$. Let $h$ be a Min-history in $\mathcal{G}_\rmdp$. Since the adversary and the environment players alternate in turns in the TBSG $\mathcal{G}_\rmdp$, the history $h$ must be of form $h  = s_0,a_0,s_1,a_1,\dots,s_{2k-2},a_{2k-2},s_{2k-1},a_{2k-1},s_{2k},a_{2k},s_{2k+1}$ for some $k \in \mathbb{N}_0$, where even-indexed states and actions belong to the adversary player and odd-indexed states and actions belong to the environment player. Let $h' = s_0,a_0,\dots,s_{2k-2},a_{2k-2},s_{2k},a_{2k}$ be the subsequence consisting only of states and actions belonging to the adversary player in $\mathcal{G}_\rmdp$. Then $h'$ defines an environment history in the RMDP, and we define
    \[ \Psi(\pi)(h) := \pi(h') \]
    where $\pi(h')$ is interpreted as a probability distribution over the vertices of the uncertainty polytope $\uncert^\rmdp(s_{2k},a_{2k})$.
    
    \paragraph{Step~2: Preservation of policy pair values} We now show that $\Phi$ and $\Psi$ preserve values of policy pairs, i.e.~that
    \begin{equation}\label{eq:goal}
        \val(\rmdp,\sigma,\pi) = \val(\mathcal{G}_\rmdp,\Phi(\sigma),\Psi(\pi))
    \end{equation}
    holds for each $\sigma \in \Sigma^{\rmdp}$ and $\pi \in \Pi^{\rmdp}$ under long-run average objective. To prove this, we first recall the definitions of values in RMDPs and TBSGs. We have
    \begin{equation*}
    \begin{split}
    &\val(\rmdp,\sigma,\pi) \\
    &= \expv^{\sigma,\pi}_{\rmdp}\Big[\limavg(S_0^\rmdp,A_0^\rmdp,S_1^\rmdp,A_1^\rmdp,\dots)\Big] \\
    &= \expv^{\sigma,\pi}_{\rmdp}\Big[\liminf_{N\rightarrow \infty} \frac{1}{N+1}\cdot \sum_{i=0}^{N}\rew(S_i^\rmdp,A_i^\rmdp)\Big],
    \end{split}
    \end{equation*}
    where we use $S_i^\rmdp$ and $A_i^\rmdp$ to denote random variables defined as the $i$-th state and action in a random trajectory induced by the probability distribution over the space of all trajectories in the RMDP defined by policies $\sigma$ and $\pi$. Similarly, by our definition of the reward function in the TBSG $\mathcal{G}_\rmdp$ which only incurs rewards in state-action pairs owned by the agent player, we have
    \begin{equation*}
    \begin{split}
    &\val(\mathcal{G}_\rmdp,\Phi(\sigma),\Psi(\pi)) \\
    &= \expv^{\Phi(\sigma),\Psi(\pi)}_{\mathcal{G}_\rmdp}\Big[\limavg(S_0^{\mathcal{G}_\rmdp},A_0^{\mathcal{G}_\rmdp},(S_0^{\mathcal{G}_\rmdp},A_0^{\mathcal{G}_\rmdp}),V_0^{\mathcal{G}_\rmdp},\dots)\Big] \\
    &= \expv^{\Phi(\sigma),\Psi(\pi)}_{\mathcal{G}_\rmdp}\Big[\liminf_{N\rightarrow \infty} \frac{1}{N+1}\cdot \sum_{i=0}^{N}\rew(S_i^{\mathcal{G}_\rmdp},A_i^{\mathcal{G}_\rmdp})\Big]
    \end{split}
    \end{equation*}
    where we use $S_i^{\mathcal{G}_\rmdp}$, $A_i^{\mathcal{G}_\rmdp}$, $V_i^{\mathcal{G}_\rmdp}$ to denote random variables defined as the $i$-th state owned by the agent player, action owned by the agent player and uncertainty polytope vertex owned by the environment player in a random trajectory induced by the probability distribution over the space of all trajectories in the TBSG defined by policies $\Phi(\sigma)$ and $\Psi(\pi)$.
    
    Hence, in order to prove eq.~\eqref{eq:goal}, by Dominated Convergence Theorem~\cite[Section~5.9]{Williams91} it suffices to show that for each $N\in\mathbb{N}$ we have
    \begin{equation}\label{eq:claim}
    \begin{split}
    &\expv^{\sigma,\pi}_{\rmdp}\Big[ \frac{\sum_{i=0}^{N}\rew(S_i^\rmdp,A_i^\rmdp)}{N+1}\Big] \\
    &= \expv^{\Phi(\sigma),\Psi(\pi)}_{\mathcal{G}_\rmdp}\Big[\frac{\sum_{i=0}^{N}\rew(S_i^{\mathcal{G}_\rmdp},A_i^{\mathcal{G}_\rmdp})}{N+1}\Big].
    \end{split}
    \end{equation}
    Eq.~\eqref{eq:claim} follows immediately by writing out the definition of the expected value and by proving that, for each $N\in\mathbb{N}$ and for each RMDP history $s_0,a_0,s_1,a_1,\dots,s_N,a_N$, we have
    \begin{equation*}
    \begin{split}
    &\probm^{\sigma,\pi}_{\rmdp}\Big[ \bigwedge_{i=0}^N S_i^\rmdp = s_i \land A_i^\rmdp = A_i \Big] \\
    &= \probm^{\Phi(\sigma),\Psi(\pi)}_{\mathcal{G}_\rmdp}\Big[\bigwedge_{i=0}^N S_i^{\mathcal{G}_\rmdp} = s_i \land A_i^{\mathcal{G}_\rmdp} = A_i\Big].
    \end{split}
    \end{equation*}
    But this can be proved straightforwardly by induction on $N$, by our construction of the TBSG $\mathcal{G}_\rmdp$ and by using Markov property in both RMDP and TBSG. Hence, eq.~\eqref{eq:goal} follows.
     
    \paragraph{Step~3: One-to-one correspondence between pure positional policies} Next, we observe that the mappings $\Phi$ and $\Psi$ become one-to-one correspondences when restricted to pure positional policies, i.e.~that
    \begin{equation*}
        \Phi_p: \Sigma^{\rmdp}_p \rightarrow \Sigma^{\mathcal{G}_\rmdp}_p, \hspace{1em}
        \Psi_p: \Pi^{\rmdp}_p \rightarrow \Pi^{\mathcal{G}_\rmdp}_p.
    \end{equation*}
    define bijections between pure positional policies in the RMDP $\rmdp$ and the TBSG $\mathcal{G}_\rmdp$. The facts that $\Phi_p$ and $\Psi_p$ are both injective and surjective follow immediately from our construction of maps $\Phi$ and $\Psi$ above.
    
    \paragraph{Step~4: Conclusion of theorem proof} We conclude the theorem claim as follows. For every pair $\sigma$ and $\pi$ of an adversary and an environment policy in the RMDP $\rmdp$, by Step~2 above we have that, under both long-run average and discounted-sum objectives,
    \[ \val(\mathcal{\rmdp},\sigma,\pi) = \val(\mathcal{G}_\rmdp,\Phi(\sigma),\Psi(\pi)). \]
    Hence, we have
    \begin{equation*}
    \begin{split}
        \val^* (\mathcal{\rmdp}) &= \sup_{\sigma \in \Sigma^{\rmdp}} \inf_{\pi \in \Pi^{\rmdp}} \val(\rmdp,\sigma,\pi) \\
        &= \sup_{\sigma \in \Sigma^{\rmdp}} \inf_{\pi \in \Pi^{\rmdp}} \val(\mathcal{G}_\rmdp,\Phi(\sigma),\Psi(\pi)) \\
        &\leq \sup_{\sigma \in \Sigma^{\rmdp}} \inf_{\pi' \in \Pi^{\mathcal{G}_\rmdp}_p} \val(\mathcal{G}_\rmdp,\Phi(\sigma),\pi') \\
        &\leq \sup_{\sigma' \in \Sigma^{\mathcal{G}_\rmdp}} \inf_{\pi' \in \Pi^{\mathcal{G}_\rmdp}_p} \val(\mathcal{G}_\rmdp,\Phi(\sigma),\pi') \\
        &= \sup_{\sigma' \in \Sigma^{\mathcal{G}_\rmdp}_p} \inf_{\pi' \in \Pi^{\mathcal{G}_\rmdp}_p} \val(\mathcal{G}_\rmdp,\Phi(\sigma),\pi') \\
        &= \val^* (\mathcal{G}_\rmdp),
    \end{split}
    \end{equation*}
    where in the second equality we used Step~2 above, in the third inequality we used the fact that pure positional policies are a subset of all policies and that $\Psi$ is a bijection when restricted to pure positional policies, in the fourth inequality we used that $\{\Phi(\sigma) \mid \sigma \in \Sigma^{\rmdp}\} \subseteq \Sigma^{\mathcal{G}_\rmdp}$ and in the last two equalities we used pure positional determinacy of TBSGs (Theorem~1 in the main paper).

    On the other hand, by pure positional determinacy of TBSGs, we also know that there exist pure positional policies $\sigma^\ast \in \Sigma^{\mathcal{G}_\rmdp}_p$ and $\pi^\ast \in \Pi^{\mathcal{G}_\rmdp}_p$ such that $\val^* (\mathcal{G}_\rmdp) = \val(\mathcal{G}_\rmdp,\sigma^\ast,\pi^\ast)$. Then, by Steps~2 and~3 above, we also have $\val^* (\mathcal{G}_\rmdp) = \val(\rmdp,\Phi^{-1}(\sigma^\ast),\Psi^{-1}(\pi^\ast))$. Hence, as above we showed that $\val^* (\mathcal{\rmdp}) \leq \val^* (\mathcal{G}^\rmdp)$ we conclude that $\Phi^{-1}(\sigma^\ast)$ and $\Psi^{-1}(\pi^\ast)$ achieve optimal payoffs in the RMDP $\rmdp$ and thus conclude that $\val^* (\mathcal{\rmdp}) = \val^* (\mathcal{G}^\rmdp)$.
\end{proof}

\section{Proof of Theorem~3}\label{app:complexityproof}

\begin{theorem*}[Complexity]
    Consider a polytopic RMDP \( \rmdp = (S^\rmdp, A^\rmdp, \uncert^\rmdp, \rew^\rmdp, s_0^\rmdp) \) and define a TBSG \( \mathcal{G}^\rmdp =  (S^{\mathcal{G}_\rmdp},A^{\mathcal{G}_\rmdp},\delta^{\mathcal{G}_\rmdp},\rew^{\mathcal{G}_\rmdp}, s_0^{\mathcal{G}_\rmdp}) \) as in Section~3 of the paper. Then, the size of $\mathcal{G}^\rmdp$ is linear in the size of $\rmdp$.
\end{theorem*}

\begin{proof}
    We need to show that each element of the tuple \( \mathcal{G}^\rmdp =  (S^{\mathcal{G}_\rmdp},A^{\mathcal{G}_\rmdp},\delta^{\mathcal{G}_\rmdp},\rew^{\mathcal{G}_\rmdp}, s_0^{\mathcal{G}_\rmdp}) \) is of size linear in the size of the elements of the tuple \( \rmdp = (S^\rmdp, A^\rmdp, \uncert^\rmdp, \rew^\rmdp, s_0^\rmdp) \).

    Before we prove this claim, recall that we assume that each polytopic RMDP is represented as an ordered tuple \( \rmdp = (S^\rmdp, A^\rmdp, \uncert^\rmdp, \rew^\rmdp, s_0^\rmdp) \), where $S^\rmdp$ is a list of states, $A^\rmdp$ is a list of actions, $\uncert^\rmdp$ is a list of polytope vertices for each state-action pair, $\rew^\rmdp$ is a list of rewards for each state-action pair and $s_0^\rmdp$ is an element of the state list.

    We now prove the desired claim:
    \begin{compactitem}
        \item {\em States.} We have $S^{\mathcal{G}_\rmdp} = S^{\mathcal{G}_\rmdp}_{\max} \cup S^{\mathcal{G}_\rmdp}_{\min}$, where $S^{\mathcal{G}_\rmdp}_{\max} = S^{\rmdp}$ and $S^{\mathcal{G}_\rmdp}_{\min} = S^{\rmdp} \times A^{\rmdp}$. Hence,
        \[ |S^{\mathcal{G}_\rmdp}| \in \mathcal{O}(|S^{\rmdp}| + |S^{\rmdp}| \cdot |A^{\rmdp}|), \]
        which is linear in the size of the RMDP, since $\mathcal{O}(|S^{\rmdp}| \cdot |A^{\rmdp}|)$ is linear in the size of $\uncert^\rmdp$. This is because $\uncert^\rmdp$ stores polytope vertices for each state-action pair in the RMDP.
        
        \item {\em Actions.} We have $A^{\mathcal{G}_\rmdp} = A^{\mathcal{G}_\rmdp}_{\max} \cup A^{\mathcal{G}_\rmdp}_{\min}$, where $A^{\mathcal{G}_\rmdp}_{\max} = A^{\rmdp}$ and $A^{\mathcal{G}_\rmdp}_{\min} = \cup_{(s,a) \in S^{\rmdp} \times A^{\rmdp}} V^{\rmdp}_{s,a}$. Hence,
        \begin{equation*}
        \begin{split}
            |A^{\mathcal{G}_\rmdp}| &\in \mathcal{O}\Big(|A^{\rmdp}| + \sum_{(s,a) \in S^{\rmdp} \times A^{\rmdp}}|V^{\rmdp}_{s,a}|\Big) \\
            &= \mathcal{O}(|A^{\rmdp}| + |\uncert^{\rmdp}|),
        \end{split}
        \end{equation*}
        which is linear in the size of the RMDP.
        
        \item {\em Transition function.} Since $\delta^{\mathcal{G}_\rmdp}: S{^\mathcal{G}_\rmdp}_{\max} \times A{^\mathcal{G}_\rmdp}_{\max} \cup S{^\mathcal{G}_\rmdp}_{\min} \times A{^\mathcal{G}_\rmdp}_{\min} \rightarrow \Delta(S{^\mathcal{G}_\rmdp})$, we have
        \begin{equation*}
        \begin{split}
            |\delta^{\mathcal{G}_\rmdp}| &= \sum_{(s,a) \in S{^\mathcal{G}_\rmdp}_{\max} \times A{^\mathcal{G}_\rmdp}_{\max}} |\delta^{\mathcal{G}_\rmdp}(s,a)| \\
            &+ \sum_{((s,a),v) \in S{^\mathcal{G}_\rmdp}_{\min} \times A{^\mathcal{G}_\rmdp}_{\min}} |\delta^{\mathcal{G}_\rmdp}((s,a),v)|.
        \end{split}
        \end{equation*}
        For each $(s,a) \in S{^\mathcal{G}_\rmdp}_{\max} \times A{^\mathcal{G}_\rmdp}_{\max}$, we have that $\delta^{\mathcal{G}_\rmdp}(s,a)$ is a Dirac probability distribution so is of constant size. On the other hand, for each $((s,a),v) \in S{^\mathcal{G}_\rmdp}_{\min} \times A{^\mathcal{G}_\rmdp}_{\min}$, we have that $\delta^{\mathcal{G}_\rmdp}((s,a),v)$ is of the same size as the probability distribution stored in the vertex $v$ of the uncertainty polytope $V^{\rmdp}_{s,a}$. Hence, the second sum is of size $\mathcal{O}(|\uncert|^\rmdp)$ as it can be written as a nested sum over state-action pairs first and then over vertices of each uncertainty polytope, and we conclude that
        \[ |\delta^{\mathcal{G}_\rmdp}| \in \mathcal{O}(|S^\rmdp| \cdot |A^\rmdp| + |\uncert^\rmdp|), \]
        which is linear in the size of the RMDP. This is because $\uncert^\rmdp$ stores polytope vertices for each state-action pair in the RMDP so $|S^\rmdp| \cdot |A^\rmdp|$ is linear in the size of $|\uncert^\rmdp|$, as already noted above.

        \item {\em Reward function.} Since $r^{\mathcal{G}_\rmdp}: S{^\mathcal{G}_\rmdp}_{\max} \times A{^\mathcal{G}_\rmdp}_{\max} \cup S{^\mathcal{G}_\rmdp}_{\min} \times A{^\mathcal{G}_\rmdp}_{\min} \rightarrow \mathbb{R}$ is defined to coincide with $2 \cdot r^\rmdp$ on state-action pairs belonging to Max and is set to be $0$ on state-action pairs belonging to Min, we conclude that
        \[ |r^{\mathcal{G}_\rmdp}| \in \mathcal{O}(|r^{\rmdp}| + |S^{\rmdp}| \cdot |A^{\rmdp}|), \]
        which is linear in the size of the RMDP, as already noted above.

        \item {\em Initial state.} Since $s_0^{\mathcal{G}_\rmdp} = s_0^{\rmdp} \in S^{\mathcal{G}_\rmdp}_{\max}$, i.e. the initial state in $\mathcal{G}^\rmdp$ is the copy of the initial state in $\rmdp$, we have
        \[ |s_0^{\mathcal{G}_\rmdp}| \in \mathcal{O}(|s_0^{\rmdp}|), \]
        which is linear in the size of the RMDP.
    \end{compactitem}
    Hence, the size of the TBSF $\mathcal{G}_\rmdp$ is linear in the size of the RMDP $\rmdp$, as claimed.
\end{proof}

\section{Extending Theorem~2 to Discounted Payoff}\label{app:discounted}

A slight modification of the reduction described in Section~3 allows reducing discounted-sum reward polytopic RMDPs to discounted-sum reward TBSGs: we just omit the doubling of reward functions and define 
\[ r^{\mathcal{G}_\rmdp}(s,a) = r^{\rmdp}(s,a). \]
for each state-action pair \(s,a\). However, the fact that two turns in \(\mathcal{G}_\rmdp\) (one turn per player) correspond to a single step in \(\rmdp\) poises another issue: in \(\mathcal{G}_\rmdp\), the discounting proceeds twice as fast compared to \(\rmdp\). A mathematically simple solution to this issue is to define the discount factor in \(\mathcal{G}_\rmdp\) to be the square root of the original discount factor \(\gamma\) from \(\mathcal{G}_\rmdp\). However, this is problematic from an algorithmic point of view, since \(\sqrt{\gamma}\) might be irrational. Hence, we propose a different modification of the reduction: instead of handling this reduction with irrational discount factor, we simply consider it as two-level discounted-sum game, where discounting happens at every alternate step. Such games have the same complexity and algorithms as discounted-sum games~\cite{ChatterjeeM12}. This gives rise to a linear-time reduction from discounted-sum reward polytopic RMDPs to discounted-sum reward TBSGs. Hence, Corollaries~1--3 hold also for RMDPs with discounted payoff objectives.

\section{Treewidth of TBSGs Reduced from RMDPs}
\label{app:treewidth}
We refer the reader to \cite{Chatterjee23} for background on treewidth and tree-decompositions. Let $\rmdp=(S^\rmdp, A^\rmdp, \uncert^\rmdp, \rew^\rmdp, s_0^\rmdp)$ be an RMDP with $m$ actions and treewidth $t$. We show that the treewidth of $\mathcal{G}_\rmdp$ is at most $(t+1)(m+1)-1$. 

\begin{proof}
    Let $T$ be a tree-decomposition of $\rmdp$ with width $t$. We construct a tree-decomposition $T'$ of $\mathcal{G}_\rmdp$ as follows: For each bag $B$ in $T$, let $B'=B \cup \{(v,a)|v \in B, a \in A^\rmdp\}$ be a bag of $T'$. Moreover, put an edge between two bags $B_1', B_2'$ in $T'$ if an only if there was an edge between $B_1, B_2$ in $T$. We verify that $T'$ is indeed a tree-decomposition of $\mathcal{G}_\rmdp$ with width $(t+1)(m+1)-1$. 

    \begin{compactitem}
        \item (Union of Bags) By construction, every node $v$ of $\rmdp$ and every node $(v,a)$ of $\mathcal{G}_\rmdp$ appear in some bag of $T'$. 
        \item(Connectivity of Node Appearance) The tree structure of $T'$ is similar to $T$. Every node $v$ of $\rmdp$ appeared in a connected sub-tree of $T$, based on how we constructed the bags of $T'$ it follows that every node $v$ appears in a connected sub-tree of $T'$ as well. Moreover, every bag that contains $v$, also contains all the pairs $(v,a)$ for $a \in A$ and every bag that contains $(v,a)$ also contains $v$. Therefore, every node $(v,a)$ appears in the same connected sub-tree of $T'$ as $v$. 
        \item (Existence of Edges) Every edge $\{v,(v,a)\}$ where $a \in A$, clearly appears in a bag containing $v$. Consider an edge of the form $\{(v,a), u\}$. Existence of such an edge implies that by action $a$ the RMDP could move from $v$ to $u$. Therefore, $v$ and $u$ were adjacent in the graph of $\rmdp$ and there exists a bag $B$ in $T$ that contains both $v,u$. By construction, the bag $B'$ contains $u$ and $(v,a)$.
        \item (Width) Let $B$ be a bag of $T$ containing at most $t+1$ nodes. For each $v \in B$, we add $m$ new nodes to $B'$. Therefore, $|B'| \leq (t+1)(m+1)$ and the width of $T'$ is at most $(t+1)(m+1)-1$.
    \end{compactitem}

\end{proof}

\end{document}